\def\year{2020}\relax
\documentclass[letterpaper]{article} % DO NOT CHANGE THIS
\usepackage{aaai20}  % DO NOT CHANGE THIS
\usepackage{times}  % DO NOT CHANGE THIS
\usepackage{helvet} % DO NOT CHANGE THIS
\usepackage{courier}  % DO NOT CHANGE THIS
\usepackage[hyphens]{url}  % DO NOT CHANGE THIS
\usepackage{graphicx} % DO NOT CHANGE THIS
\usepackage{graphicx}  % DO NOT CHANGE THIS
\usepackage[square]{natbib}
\usepackage[utf8]{inputenc} % allow utf-8 input
\usepackage[T1]{fontenc}    % use 8-bit T1 fonts
\usepackage{hyperref}

\usepackage{booktabs}       % professional-quality tables
\usepackage{amsfonts}       % blackboard math symbols
\usepackage{nicefrac}       % closed symbols for 1/2, etc.
\usepackage{microtype}      % microtypography

\usepackage{epsfig}
\usepackage{amsmath}
\usepackage{amssymb}
\usepackage{adjustbox}

\usepackage{amsthm}  
\usepackage{wasysym}
\usepackage[ruled,vlined,algo2e]{algorithm2e}
\providecommand{\SetAlgoLined}{\SetLine}
\usepackage{color}
\usepackage{dsfont}	
\usepackage{wrapfig}
\usepackage{footnote}
\usepackage{epstopdf}
\usepackage{enumitem}
\usepackage{subfigure}
\usepackage{caption}
\usepackage{comment}
\usepackage{multirow}

\usepackage{algorithm}
\usepackage{algorithmic}

\def\eg{\emph{e.g. }}
\def\ie{\emph{i.e. }}

\def\vs{\emph{vs. }}
\def\wrt{\emph{w.r.t. }}

\usepackage{pifont}% http://ctan.org/pkg/pifont

\DeclareMathOperator*{\argmin}{arg\,min}

\usepackage{mathtools}

\makeatletter
\newcommand*{\rom}[1]{\expandafter\@slowromancap\romannumeral #1@}
\makeatother
\makeatletter
\newcommand\footnoteref[1]{\protected@xdef\@thefnmark{\ref{#1}}\@footnotemark}
\makeatother
\newcommand{\bfsection}[1]{\vspace*{0.1cm}\noindent\textbf{#1.}}

\newtheorem{prop}{Proposition}

\begin{document}

% \title{White-Box Adversarial Defense: \\A Novel Perspective of Functional via Self-Supervised Data Estimation}
% \author{Anonymous Submission\\\\
% Paper ID 36
% }
% \title{Accelerating Graph Convolution via Topology-Preserving Spatial Convolution}
% \title{TP-CNNs: Topology-Preserving Convolutional Neural Networks for Learning Graph Representations}
% \title{GP-CNN: Graph-Preserving Convolutional Neural Network on 2D Grid}
\title{Graph-Preserving Grid Layout: \\A Simple Graph Drawing Method for Graph Classification using CNNs}
% \author{Anonymous Submission\\\\
% 	Paper ID 2277
% }
\author{Yecheng Lyu\textsuperscript{\rm 1}\thanks{Work was done during an internship at MERL.}, Xinming Huang\textsuperscript{\rm 1} and Ziming Zhang\textsuperscript{\rm 2}\thanks{Corresponding author.}\\
\textsuperscript{\rm 1}Worcester Polytechnic Institute, \textsuperscript{\rm 2}Mitsubishi Electric Research Laboratories (MERL)\\
\texttt{\{ylyu,xhuang\}@wpi.edu,zzhang@merl.com}
}
% \title{White-Box Adversarial Defense via Adaptive Data Recovery}  % Adaptive is too vague
%\title{Input-Dependent Defense of Adversarial Attacks by \\Recovering the Naturalness of Images}
% \author{AAAI Press\\
% Association for the Advancement of Artificial Intelligence\\
% 2275 East Bayshore Road, Suite 160\\
% Palo Alto, California 94303\\
% }
\maketitle

\begin{abstract}
	
	Graph convolutional networks (GCNs) suffer from the irregularity of graphs, while more widely-used convolutional neural networks (CNNs) benefit from regular grids. %One of the key differences between the two types of networks is the way of encoding graph topology into deep learning. 
	To bridge the gap between GCN and CNN, in contrast to previous works on generalizing the basic operations in CNNs to graph data, in this paper we address the problem of how to project undirected graphs onto the grid in a {\em principled} way where CNNs can be used as backbone for geometric deep learning. To this end, inspired by the literature of graph drawing we propose a novel graph-preserving grid layout (GPGL), an integer programming that minimizes the topological loss on the grid. Technically we propose solving GPGL approximately using a {\em regularized} Kamada-Kawai algorithm, a well-known nonconvex optimization technique in graph drawing, with a vertex separation penalty that improves the rounding performance on top of the solutions from relaxation. % when the Euclidean distance of any pair of nodes on the layout is small. 
	Using GPGL we can easily conduct data augmentation as every local minimum will lead to a grid layout for the same graph.   
	%To remedy the topological errors in GPGL, we also propose a multi-scale maxout network module as a new 2D convolutional layer and apply it to each node on the grid layout. 
	Together with the help of multi-scale maxout CNNs, we demonstrate the empirical success of our method for graph classification. % benchmarks with significant speedup in both training and inference.
	
\end{abstract}

\section{Introduction}
Graph convolutional networks (GCNs) \citep{defferrard2016convolutional, kipf2016semi, hamilton2017inductive, bronstein2017geometric, chen2018fastgcn, gao2019graph, wu2019simplifying, wu2019comprehensive, morris2019weisfeiler} refer to the family of networks that generalize well-established convolutional neural networks (CNNs) to work on structured graphs. The irregularity of a graph, \eg number of nodes and degree of each node, however, brings significant challenges in learning deep models based on graph data, especially from the perspectives of architecture design and computation. In general, a GCN has two operations to compute feature representations for the nodes in a graph, namely aggregation and transformation. That is, the feature representation of a node is computed as an aggregate of the feature representations of its neighbors before it is transformed by applying the weights and activation functions. To deal with graph irregularity the adjacency matrix is fed into the aggregation function to encode the topology. Often the weights in the transformation function are shared by all the nodes in the neighborhood to handle the node orderlessness in the convolution.

On the other hand, CNNs have gained the reputation from success in many research areas such as speech \citep{abdel2014convolutional} and vision \citep{he2016deep} due to high accuracy and high efficiency that allows millions of parameters to be trained well, thanks to the regularity of the grids leading to tensor representations. The grid topology is naturally inherent in the convolution with optimized implementation, and thus no adjacency matrix is needed. The nodes in each local neighborhood are well ordered by location, and thus sharing weights among nodes is unnecessary. 

\setlength{\columnsep}{15pt}
\begin{wrapfigure}{r}{.3\linewidth}
	\vspace{-15pt}
	\begin{center}
		\includegraphics[width=\linewidth]{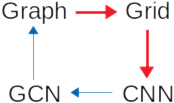}
		\vspace{-5mm}
		\caption{\footnotesize Illustration of two ways to handle graph data. The red path shows our methodology.}
		\label{fig:motivation}
	\end{center}
	\vspace{-15pt}
\end{wrapfigure}
\bfsection{Motivation}
Such key differences between graph convolution and spatial convolution make CNNs much easier and richer in deep architectures as well as being trained much more efficiently. Most of previous works in the literature such as GCNs, however, focus on the {\em blue} path in Fig. \ref{fig:motivation} that aim to generalize the basic operations in CNNs and apply the new operations to graph data at the cost of, for instance, higher running time. Intuitively there also exists another way to handle graph data, shown as the {\em red} path in Fig. \ref{fig:motivation}, that maps graphs onto grids so that CNNs can be applied directly to graph based applications. Such methods can easily inherit the benefits from CNNs such as computational efficiency. We notice, however, that there are few papers in the literature to explore this methodology for geometric deep learning (GDL). As a graph topology can be richer and more flexible than a grid, and often encode important message of the graph, how to preserve graph topology on the grid becomes a key challenge for such methods.

Therefore, in order to bridge the gap between GCN and CNN and in contrast to previous works on generalizing the basic operations in CNNs to graph input, in this paper we aim to address the following question: {\em Can we use CNNs as backbone for GDL by projecting undirected graphs onto the grid in a {\em principled} way to preserve graph topology?}

% \begin{figure}[t]
%     \centering
%     \includegraphics[width=\linewidth]{comp.pdf}
%     \vspace{-6mm}
%     \caption{\footnotesize Illustration of differences between {\bf (a)} graph, {\bf (b)} our graph-preserving grid layout (GPGL) and {\bf (c)} image. The black color in (b) denotes no correspondence to any node in (a). %Similarly, graph pooling can be replaced as well by 2D pooling on both feature map and binary mask.
%     }
%     \label{fig:overview}
%     \vspace{-5mm}
% \end{figure}

In fact, how to visualize structural information as graphs has been well studied in {\em graph drawing}, an area of mathematics and computer science combining methods from geometric graph theory and information visualization to derive 2D or 3D depictions of graphs with vertices and edges whose arrangement within a drawing affects its understandability, usability, fabrication cost, and aesthetics \citep{di1994algorithms}. In the literature, the {\em Kamada-Kawai (KK)} algorithm \citep{kamada1989algorithm} is one of the most widely-used undirected graph visualization techniques. In general, the KK algorithm defines an objective function that measures the energy of each graph layout, and searches for the (local) minimum. This often leads to a graph layout where adjacent nodes are near some pre-specified distance from each other, and non-adjacent nodes are well-spaced.

% In the context of facilitating GCN %, an important research area in geometric deep learning which studies the extension of deep learning techniques to graph/manifold structured data, 
% using CNNs we are facing some challenges. %as illustrated in Fig.~\ref{fig:overview}.

To the best of our knowledge, however, such graph drawing algorithms have never been explored for GDL. One possible reason is that graph drawing algorithms often work in continuous spaces, while our case requires {\em discrete} spaces (\ie grid) where CNNs can be deployed. Overall, how to project graphs onto the grid with topology preservation for GDL is still elusive in the literature.

\begin{figure}[t]
	\centering
	\includegraphics[width=\linewidth]{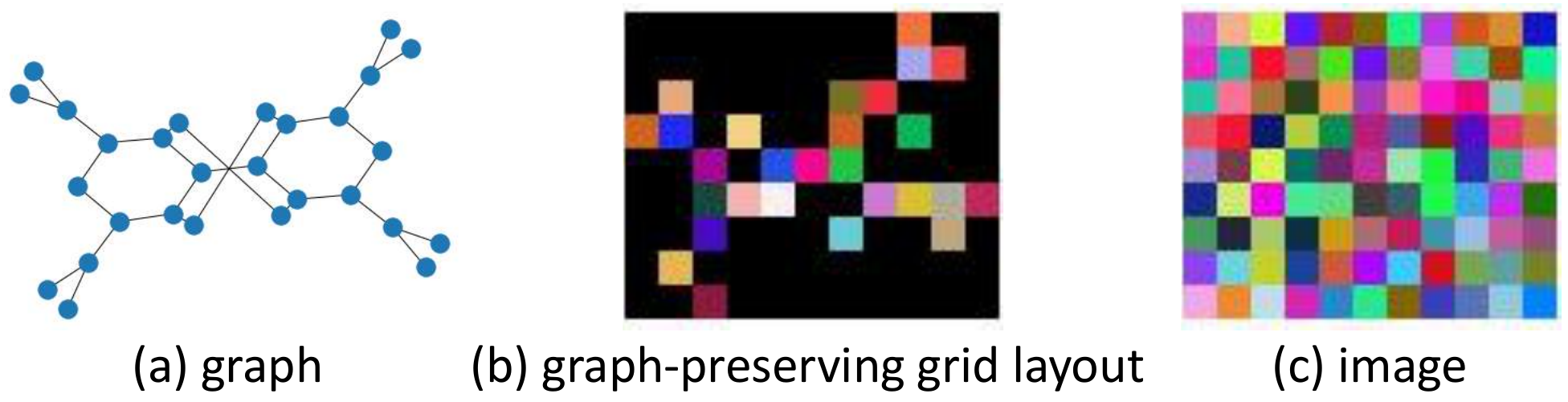}
	\vspace{-6mm}
	\caption{\footnotesize Illustration of differences between {\bf (a)} graph, {\bf (b)} our graph-preserving grid layout (GPGL) and {\bf (c)} image. The black color in (b) denotes no correspondence to any vertex in (a), and other colors denote non-zero features on the grid vertices. %Similarly, graph pooling can be replaced as well by 2D pooling on both feature map and binary mask.
	}
	\label{fig:overview}
	\vspace{-3mm}
\end{figure}

\bfsection{Contributions}
To address the question above, in this paper we propose a novel {\em graph-preserving grid layout} (GPGL), an integer programming problem that minimizes the topological loss on the grid so that CNNs can be used for GDL on undirected graphs. Technically solving such a problem is very challenging because potentially one needs to solve a highly nonconvex optimization problem in a discrete space. We manage to do so by proposing a regularized KK method with a novel vertex separation penalty, followed by the rounding technique. As a result, our GPGL can manage to preserve the irregular structural information in a graph on the regular grid as graph layout, as illustrated in Fig. \ref{fig:overview}. 

In summary, our key contributions are twofold as follows:
\begin{itemize}
	\item We are the {\em first}, to the best of our knowledge, to explicitly explore the usage of graph drawing algorithms in the context of GDL, and accordingly propose a novel regularized Kamada-Kawai algorithm to project graphs onto the grid with minimum loss in topological information.
	\item We demonstrate the empirical success of GPGL on graph classification, with the help of affine-invariant CNNs.
\end{itemize}

\section{Related Work}
\bfsection{Graph Drawing \& Network Embedding}
Roughly speaking, graph drawing can be considered as a subdiscipline of network embedding \citep{hamilton2017representation, cui2018survey, cai2018comprehensive} whose goal is to find a low dimensional representation of the network nodes in some metric space so that the given similarity (or distance) function is preserved as much as possible. In summary, graph drawing focuses on the 2D/3D visualization of graphs \citep{dougrusoz1996circular, eiglsperger2001orthogonal, koren2005drawing, spielman2007spectral, tamassia2013handbook}, while network embedding emphasizes the learning of low dimensional graph representations. Despite the research goal, similar methodology has been applied to both areas. For instance, the KK algorithm \citep{kamada1989algorithm} was proposed for graph visualization as a force-based layout system with advantages such as good-quality results and strong theoretical foundations, but suffering from high computational cost and poor local minima. Similarly \cite{tenenbaum2000global} proposed a global geometric framework for network embedding to preserve the intrinsic geometry of the data as captured in the geodesic manifold distances between all pairs of data points. There are also some works on drawing graphs on lattice, \eg \citep{freese2004automated}.

In contrast to graph drawing, our focus is to project an existing graph onto the grid with minimum topological loss so that CNNs can be deployed efficiently and effectively to handle graph data. In such a context of GDL, we are not aware of any work that utilizes the graph drawing algorithms to facilitate GDL, to the best of our knowledge.

\bfsection{Graph Synthesis \& Generation}
Methods in this field, \eg \citep{grover2018graphite, li2018learning, you2018graphrnn, samanta2019nevae}, often aim to learn a (sophisticated) generative model that reflects the properties of the training graphs. Recently, \cite{kwon2019deep} proposed learning an encoder-decoder for the graph layout generation problem to systematically visualize a graph in diverse layouts using deep generative model. \cite{franceschi2019learning} proposed jointly learning the graph structure and the parameters of GCNs by approximately solving a bilevel program that learns a discrete probability distribution on the edges of the graph for classification problems.

In contrast to such methods above, our algorithm for GPGL is essentially a self-supervised learning algorithm that is performed for each individual graph and requires no training at all. Moreover, we focus on re-deploy each graph onto the grid as layout while preserving its topology. This procedure is separate from the training of CNNs later.

\begin{figure*}[t]
	% 	\begin{minipage}[b]{0.195\textwidth}
	% 		\begin{center}
	% 			\centerline{\includegraphics[width=\columnwidth]{original_draw_spring.png}}
	% 			\centerline{\footnotesize (a) Fully-connected graph}
	% 		\end{center}
	% 	\end{minipage}
	% 	\hfill
	\begin{minipage}[b]{0.245\textwidth}
		\begin{center}
			\centerline{\includegraphics[width=\columnwidth]{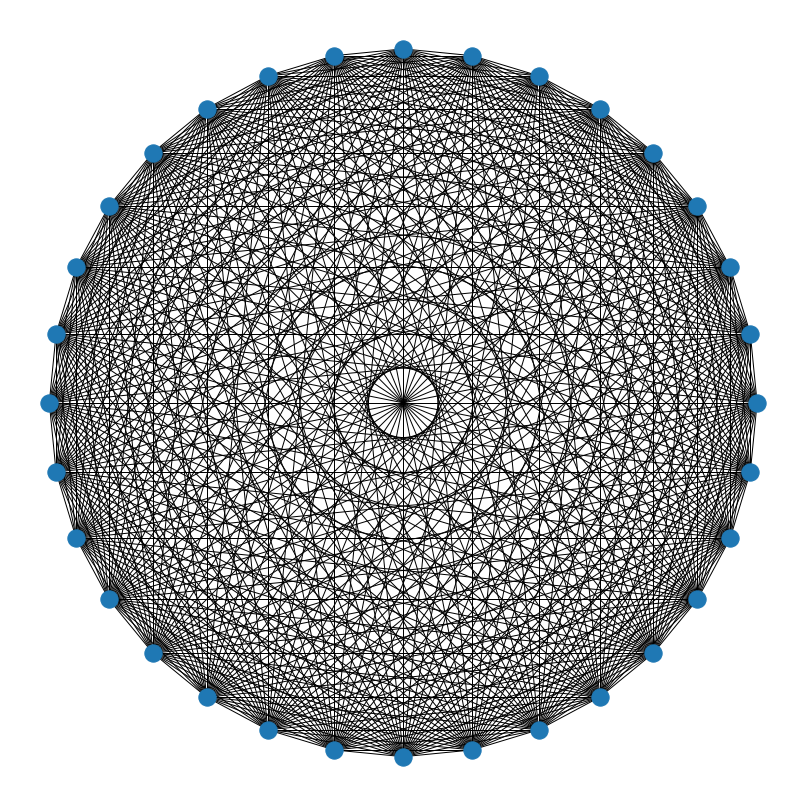}}
			\centerline{\footnotesize (a) KK before rounding}
		\end{center}
	\end{minipage}
	\hfill
	\begin{minipage}[b]{0.245\textwidth}
		\begin{center}
			\centerline{\includegraphics[width=\columnwidth]{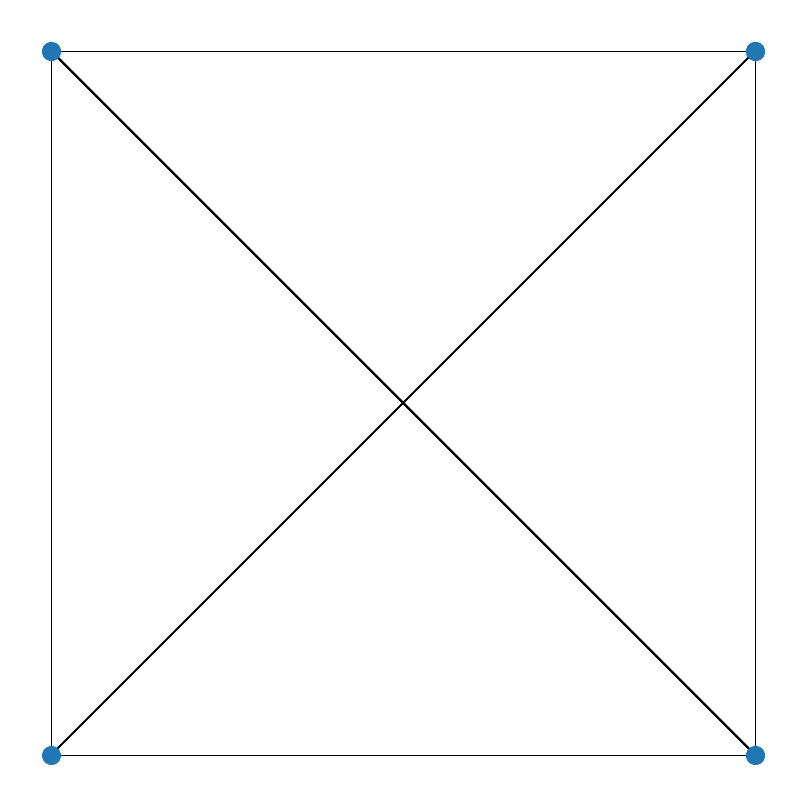}}
			\centerline{\footnotesize (b) KK after rounding}
		\end{center}
	\end{minipage}
	\hfill
	\begin{minipage}[b]{0.245\textwidth}
		\begin{center}
			\centerline{\includegraphics[width=\columnwidth]{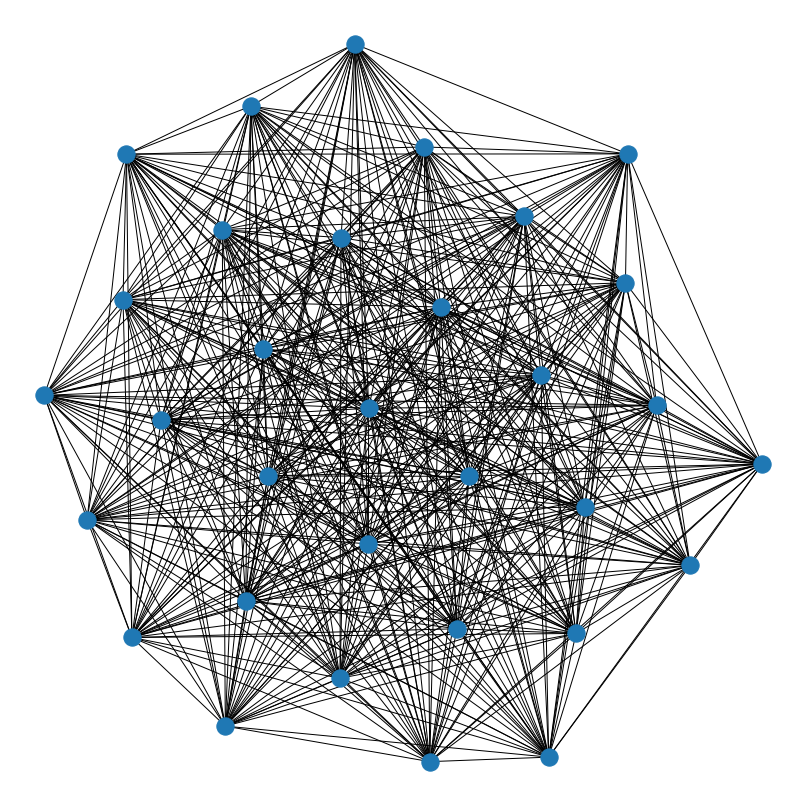}}
			\centerline{\footnotesize (c) Ours before rounding}
		\end{center}
	\end{minipage}
	\hfill
	\begin{minipage}[b]{0.245\textwidth}
		\begin{center}
			\centerline{\includegraphics[width=\columnwidth]{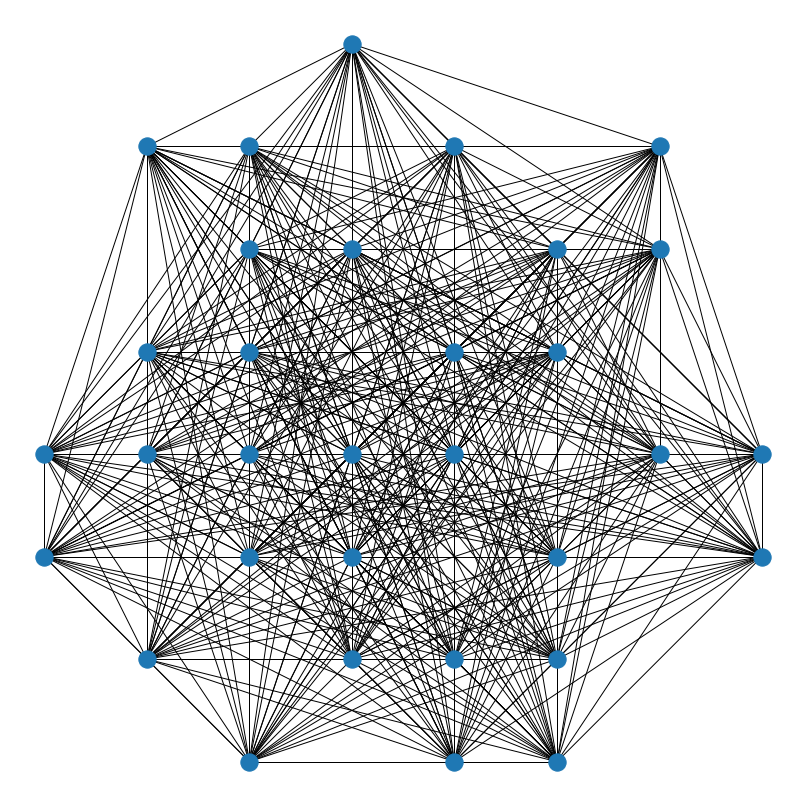}}
			\centerline{\footnotesize (d) Ours after rounding}
		\end{center}
	\end{minipage}
	\hfill
	\vspace{-5mm}
	\caption{\footnotesize Illustration of layout comparison between the KK algorithm and our proposed regularized KK algorithm before and after rounding based on a fully-connected graph with 32 vertices.}
	\label{fig:drawing_comparison}
	\vspace{-3mm}
\end{figure*}

\bfsection{Geometric Deep Learning}
In general GDL studies the extension of deep learning techniques to graph and manifold structured data (\eg \citep{kipf2016semi, bronstein2017geometric, monti2017geometric, huang2018building}). In particular in this paper we focus on graph data only. Broadly GDL methods can be categorized into spatial methods (\eg \citep{masci2015geodesic, boscaini2016learning, monti2017geometric}) and spectral methods (\eg \citep{defferrard2016convolutional, levie2018cayleynets, yi2017syncspeccnn}). Some nice survey on this topic can be found in \citep{bronstein2017geometric, zhang2018deep, hamilton2017representation, wu2019comprehensive}.

\cite{niepert2016learning} proposed a framework for learning convolutional neural networks by applying the convolution operations to the locally connected regions from graphs. We are different from such a work by applying CNNs to the grids where graphs are projected to with topology preservation.
%\cite{schutt2017schnet} proposed using continuous-filter convolutional layers to model quantum interactions in molecules without requiring the data to lie on a grid. 
\cite{tixier2017graph} proposed an ad-hoc method to project graphs onto 2D grid and utilized CNNs for graph classification. Specifically each node is embedded into a high dimensional space, then mapped to 2D space by PCA, and finally quantized into grid. In contrast, we propose a principled and systematical way based on graph drawing, \ie a nonconvex integer programming formulation, for mapping graphs onto 2D grid. Besides, our graph classification performance is much better than both works. On MUTAG and IMDB-B data sets we can achieve 90.42\% and 74.9\% test accuracy with 1.47\% improvement over \citep{niepert2016learning} and 4.5\% improvement over \citep{tixier2017graph}, respectively.

% \bfsection{Integer Programming}

\section{Graph-Preserving Grid Layout}
% \subsection{Approximately Topology-Preserving Mapping}

\subsection{Problem Setup}
Let $\mathcal{G}=(\mathcal{V}, \mathcal{E})$ be an undirected graph with a vertex set $\mathcal{V}$ and an edge set $\mathcal{E}\subseteq\mathcal{V}\times\mathcal{V}$, and $s_{ij}\geq1, \forall i\neq j$ be the graph-theoretic distance such as shortest-path between two vertices $v_i, v_j\in\mathcal{V}$ on the graph that encodes the graph topology. 

Now we would like to learn a function $f:\mathcal{V}\rightarrow\mathbb{Z}^2$ to map the graph vertex set to a set of 2D integer coordinates on the grid so that the graph topology can be preserved as much as possible given a metric $d:\mathbb{R}^2\times\mathbb{R}^2\rightarrow\mathbb{R}$ and a loss $\ell:\mathbb{R}\times\mathbb{R}\rightarrow\mathbb{R}$. As a result, we are seeking for $f$ to minimize the following objective:
\begin{align}\label{eqn:general_formula}
\min_{f}\sum_{i\neq j}\ell(d(f(v_i), f(v_j)), s_{ij}).
\end{align}
Now letting $\mathbf{x}_i=f(v_i)\in\mathbb{Z}^2$ as reparametrization, we can rewrite Eq. \ref{eqn:general_formula} as the following {\em integer programming} problem:
\begin{align}\label{eqn:int_pro}
\min_{\mathcal{X}\subseteq\mathbb{Z}^2}\sum_{i\neq j}\ell(d(\mathbf{x}_i, \mathbf{x}_j), s_{ij}),
\end{align}
where the set $\mathcal{X}=\{\mathbf{x}_i\}$ denotes the {\em 2D grid layout} of the graph, \ie all the vertex coordinates on the 2D grid. %Regarding the formulation, we have the following remarks:

% \noindent
% \underline{\em (R1) Self-Supervision:}
\bfsection{Self-Supervision}
Note that the problem in Eq. \ref{eqn:int_pro} needs to be solved for each individual graph, which is related to self-supervision as a form of unsupervised learning where the data itself provides the supervision \citep{zisserman2018}. This property is beneficial for data augmentation, as every local minimum will lead to a grid layout for the same graph.

% \noindent
% \underline{\em (R2) Only 2D Grid Layout:} 
\bfsection{2D Grid Layout}
In this paper we are interested in learning only 2D grid layouts for graphs, rather than higher dimensional grids (even 3D) where we expect that the layouts would be more compact in volume and would have larger variance in configuration, both bringing more challenges into training CNNs properly later. We confirm our hypothesis based on empirical observations. Besides, the implementation of 3D basic operations in CNNs such as convolution and pooling are often slower than 2D counterparts, and the operations beyond 3D are not available publicly. 

% \noindent
% \underline{\em (R3) Relaxation \& Rounding for Integer Programming:} 
\bfsection{Relaxation \& Rounding for Integer Programming}
Integer programming is NP-complete and thus finding exact solutions is challenging, in general \citep{wolsey2014integer}. Relaxation and rounding is a widely used heuristic for solving integer programming due to its efficiency \citep{bradley1977applied}, where the rounding operator is applied to the solution from the real-number relaxed problem as the solution for the integer programming. In this paper we employ this heuristic to learn 2D grid layouts. For simplicity, in the sequel we will only discuss how to solve the relaxation problem (\ie before rounding).

\subsection{Regularized Kamada-Kawai Algorithm}
In this paper we set $\ell$ and $d$ in Eq. \ref{eqn:int_pro} to the least-square loss and Euclidean distance to preserve topology, respectively. 

\bfsection{Kamada-Kawai Algorithm}
The KK graph drawing algorithm \citep{kamada1989algorithm} was designed for a (relaxed) problem in Eq. \ref{eqn:int_pro} with a specific objective function as follows:
\begin{align}\label{eqn:kk}
% \text{\bf Kamada-Kawai: } \;
\min_{\mathcal{X}\subseteq\mathbb{R}^2}\mathcal{L}_{KK} = \sum_{i\neq j}\frac{1}{2}\left(\frac{d_{ij}}{s_{ij}} - 1\right)^2,
\end{align}
where $d_{ij}=\|\mathbf{x}_i - \mathbf{x}_j\|, \forall (i,j)$ denotes the Euclidean distance between vertices $v_i$ and $v_j$. Note that there is no regularization to control the distribution of nodes in 2D visualization. 

Fig. \ref{fig:drawing_comparison} illustrates the problems using the KK algorithm when projecting the fully-connected graph onto 2D grid. Eventually KK learns a circular distribution with equal space among the vertices as in Fig. \ref{fig:drawing_comparison}(a) to minimize the topology preserving loss in Eq.~\ref{eqn:kk}. When taking a close look at these 2D locations we find that after transformation all these locations are within the square area $[0,1]\times[0,1]$, leading to the square pattern in Fig. \ref{fig:drawing_comparison}(b) after rounding. Such behavior totally makes sense to KK because it does not care about the grid layout but only the topology preserving loss. However, our goal is not only to preserve the topology but also to make graphs visible on the 2D grid in terms of vertices.

\bfsection{Vertex Separation Penalty}
To this end, we propose a novel vertex separation penalty to regularize the vertex distribution on the grid. The intuition behind it is that when the minimum distance among all the vertex pairs is larger than a threshold, say 1, it will guarantee that after rounding every vertex will be mapped to a unique 2D location with no overlap. But when any distance is smaller than the threshold, it should be considered to enlarge the distance, otherwise, no penalty. Moreover, we expect that the penalties will grow faster than the change of distances and in such a way the vertices can be re-distributed more rapidly. Based on these considerations we propose the following regularizer:
\begin{align}\label{eqn:r_sep}
\mathcal{R}_{sep} = \lambda\sum_{i\neq j} \max\left\{0, \frac{\alpha}{d_{ij}} - 1\right\}, 
\end{align}
where $\alpha\geq0, \lambda\geq0$ are two predefined constants. From the gradient of $\mathcal{R}_{sep}$ \wrt an arbitrary 2D variable $\mathbf{x}_i$, that is,
\begin{align}
\frac{\partial\mathcal{R}_{sep}}{\partial\mathbf{x}_i} = -\lambda\sum_{i\neq j} \frac{\mathbf{x}_i - \mathbf{x}_j}{d_{ij}^3}\mathbf{1}_{\{d_{ij}<\alpha\}}
\end{align}
where $\mathbf{1}_{\{\cdot\}}$ denotes the indicator function returning 1 if the condition is true, otherwise 0, we can clearly see that $\alpha$ as a threshold controls when penalties occur, and $\lambda$ controls the trade-off between the loss and the regularization, leading to different step sizes in gradient based optimization.

\begin{algorithm}[t]\footnotesize
	\SetAlgoLined
	\SetKwInOut{Input}{Input}\SetKwInOut{Output}{Output}
	\Input{graph distance $\{s_{ij}\}$, parameters $\alpha, \lambda, \gamma$}
	\Output{2D grid layout $\mathcal{X}^*$}
	\BlankLine
	
	$\tilde{\mathcal{X}}\leftarrow\argmin_{\mathcal{X}}\mathcal{L}_{KK}$ with a (randomly shuffled) circular layout;
	
	$\beta\leftarrow\max(\gamma, \min_{\forall \mathbf{x}_i, \mathbf{x}_j\in\tilde{\mathcal{X}},  i\neq j}\|\mathbf{x}_i - \mathbf{x}_j\|)$;
	
	$\mathbf{x}_i\leftarrow\max(1, \frac{\alpha}{\beta})\cdot\mathbf{x}_i, \forall \mathbf{x}_i\in\tilde{\mathcal{X}}$\tcp*{optional: rescaling}
	
	$\mathcal{X}^*\leftarrow\argmin_{\mathcal{X}\subseteq\mathbb{R}^2}\mathcal{L}_{GPGL}$ with set $\{\mathbf{x}_i\}$ as initialization;
	
	$\mathcal{X}^*\leftarrow\mbox{rounding}(\mathcal{X}^*)$;
	
	\Return $\mathcal{X}^*$;
	\caption{Regularized Kamada-Kawai Algorithm}\label{alg:RKK_alg}
\end{algorithm}

\bfsection{Our Algorithm}
With the help of our new regularization, we propose a regularized KK algorithm as listed in Alg. \ref{alg:RKK_alg}, which tries to minimize the following regularized KK loss:
\begin{align}
    \min_{\mathcal{\mathcal{X}}\subseteq\mathbb{Z}^2}\mathcal{L}_{GPGL} = \mathcal{L}_{KK} + \mathcal{R}_{sep}.
\end{align}

\noindent
\underline{\em (1) Initialization:}
Note that both KK and our algorithms are highly nonconvex, and thus good initialization is need to make both work well, \ie convergence to good local minima. 

To this end, we first utilize the KK algorithm to generate a vertex distribution. To do so, we employ the implementation in the Python library {\sc NetworkX} which uses a circular layout as initialization by default. As discussed above, KK has no control on the vertex distribution. This may lead to serious {\bf vertex loss} problems in the 2D grid layout where some of vertices in the original graph merge together as a single vertex on the grid after rounding due to small distances. %(see our experiments).

To overcome this problem, we introduce an optional step, {\em rescaling}, to enlarge the pairwise distances linearly. %so that the minimum is no less than the threshold $\alpha$ in Eq. \ref{eqn:r_sep} and thus no penalty is induced. 
Intuitively this step ``zooms'' in the vertex distribution returned by KK when the minimum distance is smaller than the threshold but cannot be too small (controlled by parameter $\gamma\geq0$). In this way not only such vertices are still distributed well, but also there will be enough space among the vertices to search for a good local minimum for our algorithm.

By taking the scaled solution as initialization, we minimize the regularized KK loss, where the KK objective is used to preserve the graph topology and the regularizer monitors the distances to prevent the vertex loss on the grid. %A 2D grid layout for a graph is generated after applying the rounding to the minimizer.

\noindent
\underline{\em (2) Topology Preservation with Regularization:}
As we observe, the key challenge in topology preservation comes from the node degree, and the lower degree the easier for preservation. Since there are only 8 neighbors at most in the 2D grid layout, it will induce a penalty for a graph vertex whose degree is higher than 8. Fig. \ref{fig:drawing_comparison} illustrates such a case where the original graph is full-connected with 32 vertices. With the help of our regularization, we manage to map this graph to a ball-like grid layout, as shown in Fig. \ref{fig:drawing_comparison}(c) and (d). %This totally makes sense because in order to minimize the regularized KK loss, the distribution of vertices on the grid has to be ball-like.
\begin{prop}\label{prop:1}
	An ideal 2D grid layout with no vertex loss for a full-connected graph with $|\mathcal{V}|$ vertices is a ball-like shape with radius of $\lceil (\frac{|\mathcal{V}|}{\pi})^{\frac{1}{2}}\rceil$ that minimizes Eq.~\ref{eqn:int_pro} with relaxation of the regularized Kamada-Kawai loss. Here $\lceil\cdot\rceil$ denotes the ceiling operation.
\end{prop}
%\vspace{-5mm}
\begin{proof}
	Given the conditions in the proposition above, we have $s_{ij}=1, d_{ij}\geq 1,  \forall i\neq j$ and $\mathcal{R}_{sep}=0$. Without loss of generalization, we uniformly deploy the graph vertices in a circle and set the circular center $A$ to a node on the 2D grid. Now imagine the gradient field over all the vertices as a sandglass centered at $A$ where each vertex is a ball with a unit diameter. Then it is easy to see that by the ``gravity'' (\ie gradient) all the vertices move towards the center $A$, and eventually are stabilized (as a local minimum) within an $r$-radius circle whose covering area should satisfy $|\mathcal{V}|\leq \pi r^2$, \ie $r=\lceil (\frac{|\mathcal{V}|}{\pi})^{\frac{1}{2}}\rceil$ as the smallest sufficient radius to cover all the vertices with guarantee. We now complete our proof.
\end{proof}
%\vspace{-3mm}

Note that Fig. \ref{fig:drawing_comparison}(d) exactly verifies Prop. \ref{prop:1} with a radius $r=\lceil (\frac{32}{\pi})^{\frac{1}{2}}\rceil=4$. In summary, our algorithm can manage to preserve graph topology on the 2D grid even when the node degree is higher than 8.

\noindent
\underline{\em (3) Computational Complexity:} 
The KK algorithm has the computational complexity of, at least, $O(|\mathcal{V}|^2)$ \citep{kobourov2012spring} that limits the usage of KK to medium-size graphs (\eg 50-500 vertices). Since our algorithm in Alg. \ref{alg:RKK_alg} is based on KK, it unfortunately inherits this limitation as well. To accelerate the computation for large-scale graphs, we potentially can adopt the strategy in multi-scale graph drawing algorithms such as \citep{harel2000fast}. However, such an extension is out of scope of this paper, and we will consider it in our future work.

\subsection{Graph Classification}
\bfsection{Data Augmentation}
As mentioned in self-supervision, each local minimum from our regularized KK algorithm in Alg. \ref{alg:RKK_alg} will lead to a grid layout for the graph, while each minimum depends on its initialization. Therefore, to augment grid layout data from graphs, we do a random shuffle on the circular layout when applying Alg. \ref{alg:RKK_alg} to an individual graph. 

\bfsection{Grid-Layout based 3D Representation}
Once a grid layout is generated, we first crop the layout with a sufficiently large fixed-size window (\eg $64\times 64$), and then associate each vertex feature vector from the graph with the projected node within the window. All the layouts are aligned to the top-left corner of the window. The rest of nodes in the window with no association of feature vectors are assigned to zero vectors. 

Once vertex loss occurs, we take an average, by default, of all the vertex feature vectors (\ie average-pooling) and assign it to the grid node. We also compare average-pooling with max-pooling for merging vertices, and observe similar performance empirically in terms of classification.

\begin{figure}[t]
	\centering
	\includegraphics[width=\linewidth]{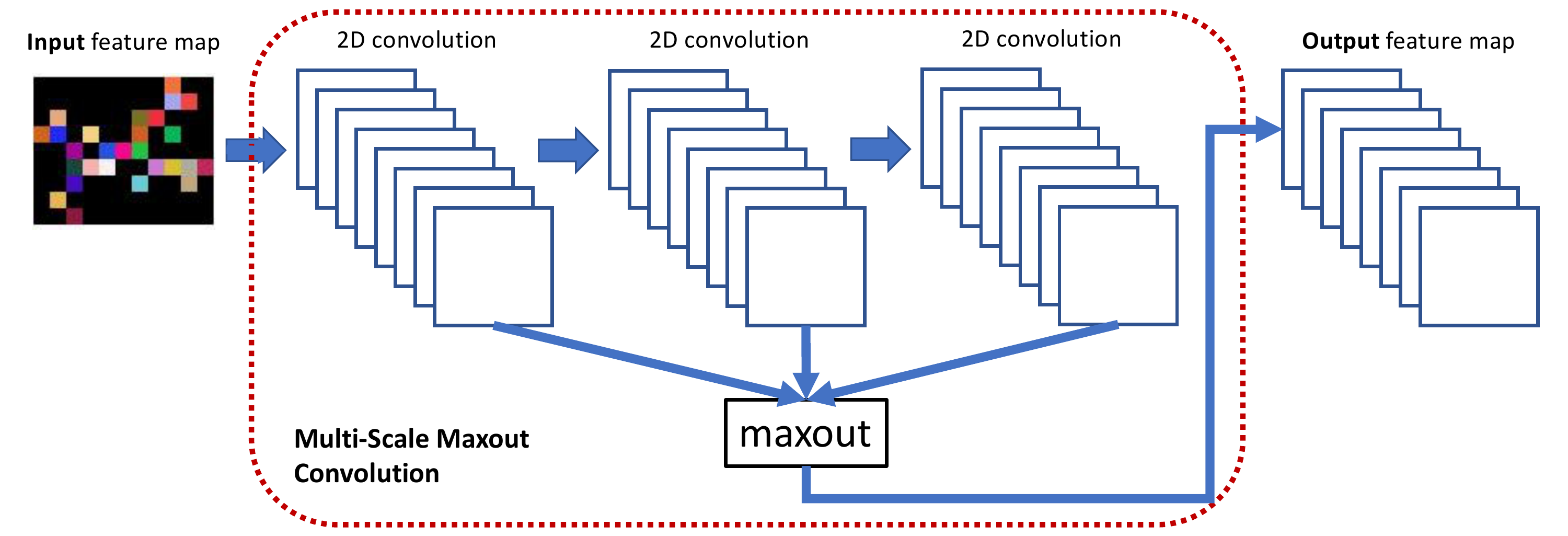}
	\vspace{-6mm}
	\caption{\footnotesize Multi-scale maxout convolution (MSM-Conv).
	}
	\label{fig:mm_cnn}
	\vspace{-3mm}
\end{figure}

\begin{figure*}[t]
	\begin{minipage}[b]{0.195\textwidth}
		\begin{center}
			\centerline{\includegraphics[width=\columnwidth]{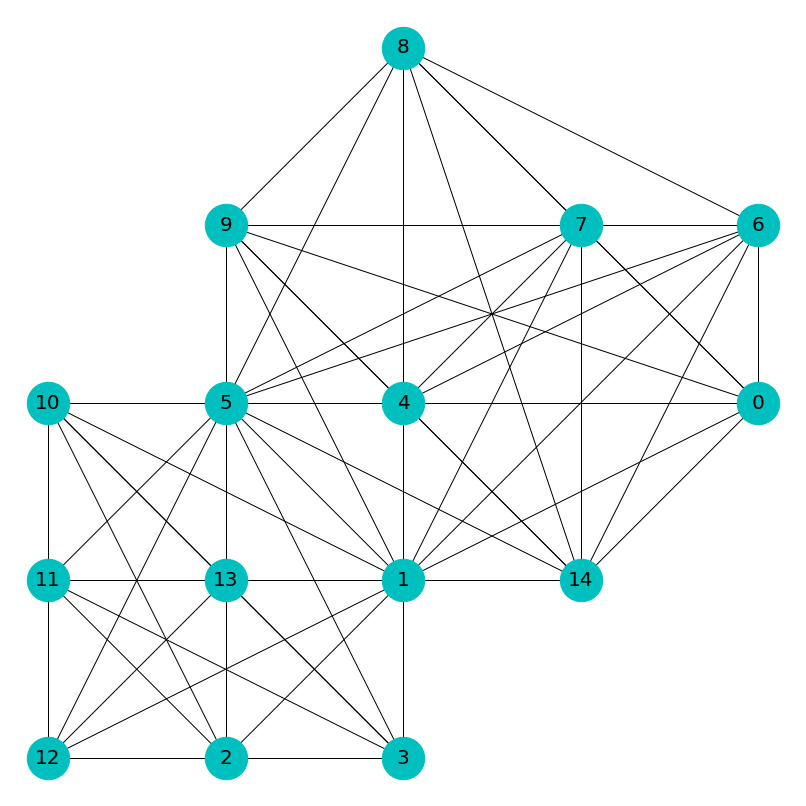}}
			\centerline{\footnotesize (a) $\alpha=1.00, \lambda=1000$}
		\end{center}
	\end{minipage}
	\hfill
	\begin{minipage}[b]{0.195\textwidth}
		\begin{center}
			\centerline{\includegraphics[width=\columnwidth]{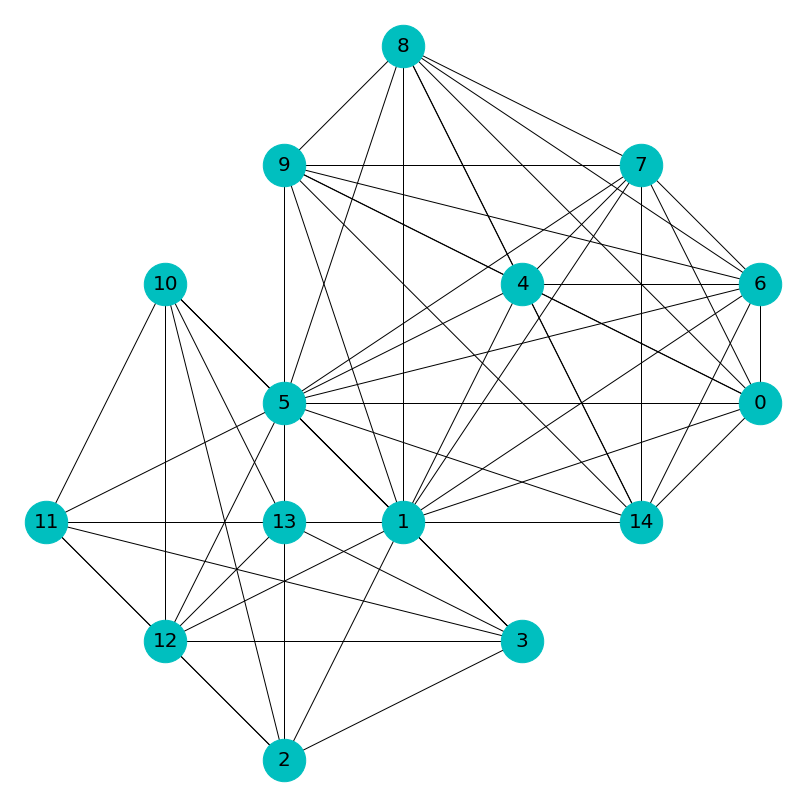}}
			\centerline{\footnotesize (b) $\alpha=1.50, \lambda=1000$}
		\end{center}
	\end{minipage}
	\hfill
	\begin{minipage}[b]{0.195\textwidth}
		\begin{center}
			\centerline{\includegraphics[width=\columnwidth]{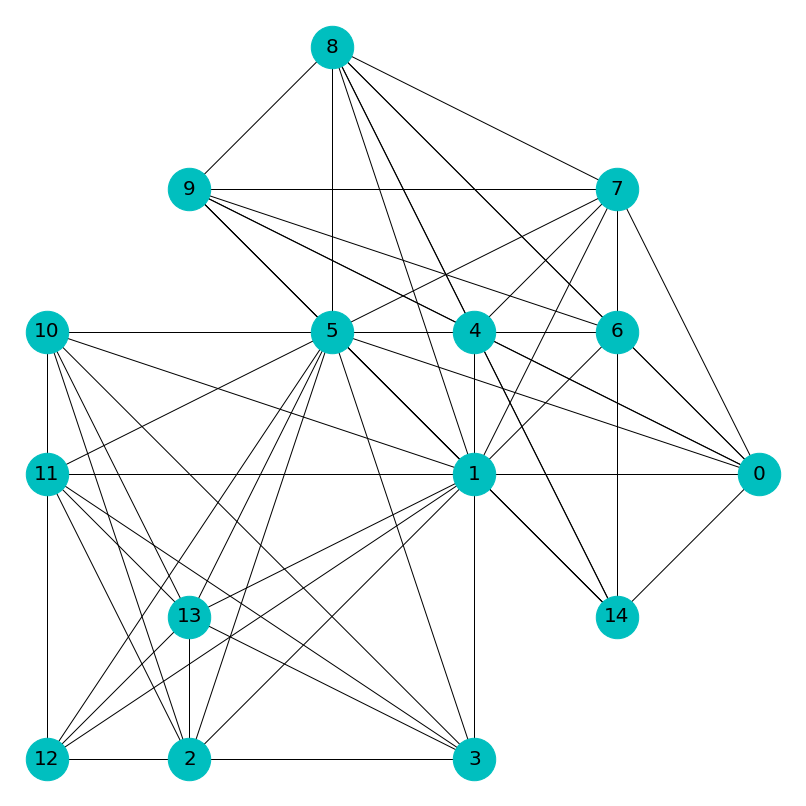}}
			\centerline{\footnotesize (c) $\alpha=1.25, \lambda=1000$}
		\end{center}
	\end{minipage}
	\hfill
	\begin{minipage}[b]{0.195\textwidth}
		\begin{center}
			\centerline{\includegraphics[width=\columnwidth]{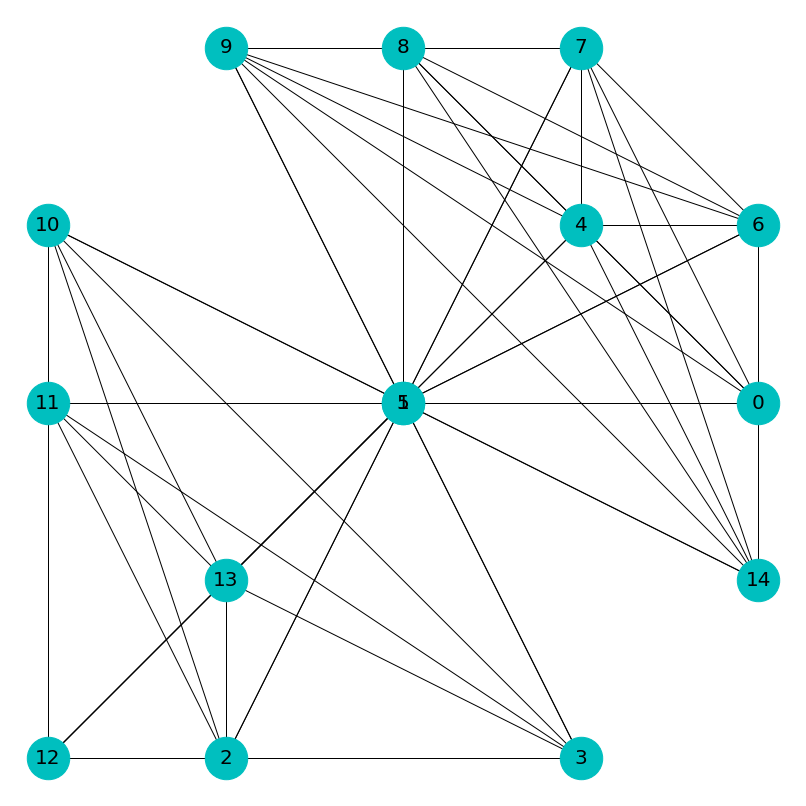}}
			\centerline{\footnotesize (d) $\alpha=1.25, \lambda=200$}
		\end{center}
	\end{minipage}
	\hfill
	\begin{minipage}[b]{0.195\textwidth}
		\begin{center}
			\centerline{\includegraphics[width=\columnwidth]{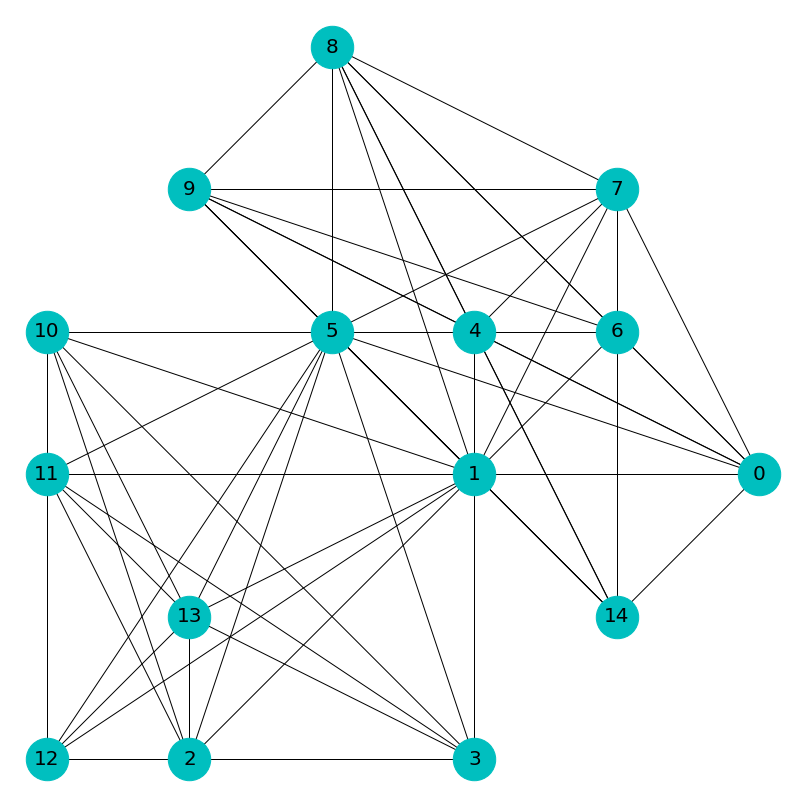}}
			\centerline{\footnotesize (e) $\alpha=1.25, \lambda=5000$}
		\end{center}
	\end{minipage}
	\hfill
	\vspace{-5mm}
	\caption{\footnotesize Illustration of effects of different combinations of $\alpha, \lambda$ on grid layout generation (IMDB-B)}
	\label{fig:alpha_lambda}
	% \vspace{-5mm}
\end{figure*}
\begin{figure*}[t]
	\begin{minipage}[b]{0.325\textwidth}
		\begin{center}
			\centerline{\includegraphics[width=\columnwidth]{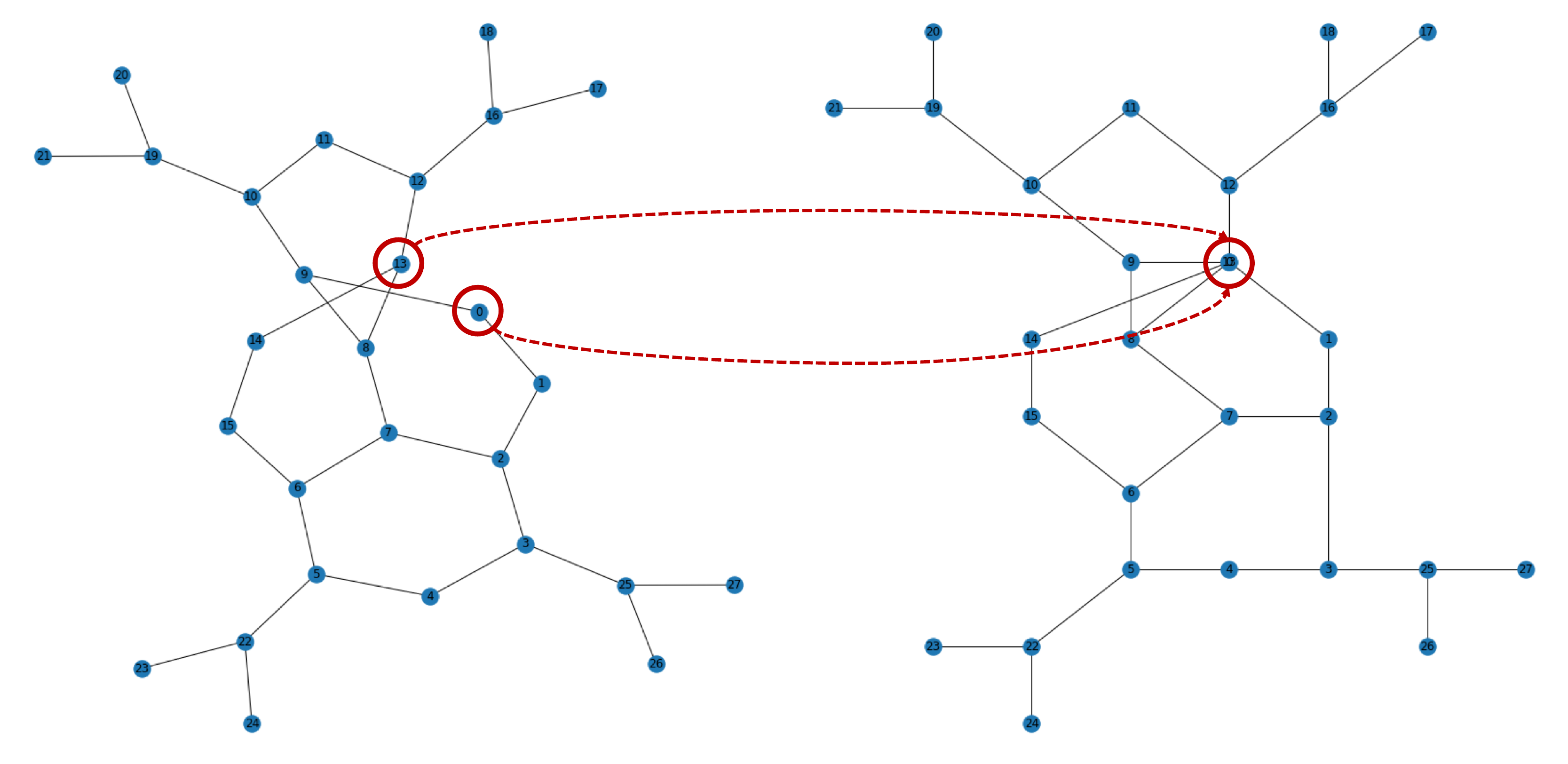}}
			\centerline{\footnotesize (a) MUTAG}
		\end{center}
	\end{minipage}
	\hfill
	\begin{minipage}[b]{0.325\textwidth}
		\begin{center}
			\centerline{\includegraphics[width=\columnwidth]{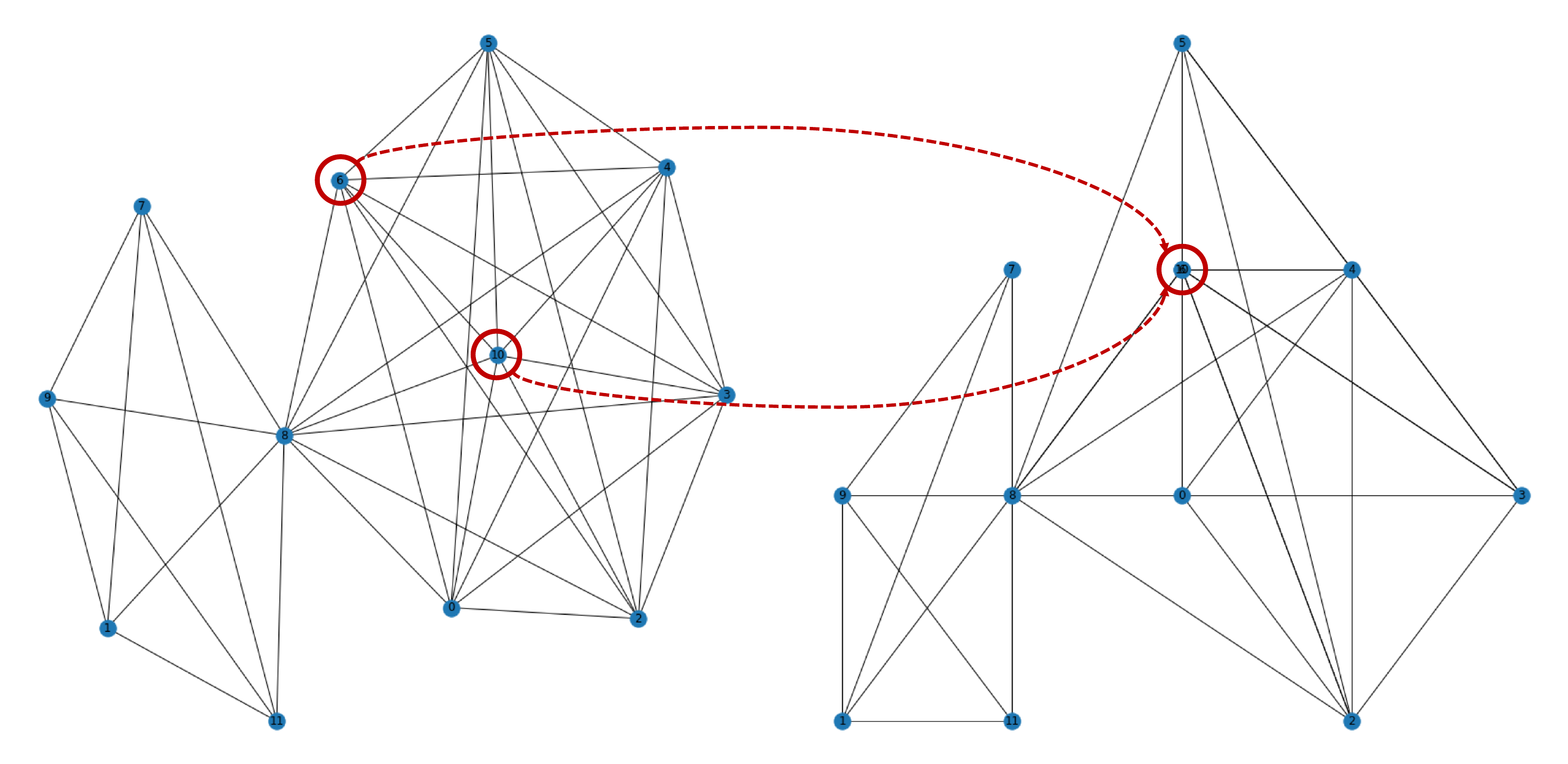}}
			\centerline{\footnotesize (b) IMDB-B}
		\end{center}
	\end{minipage}
	\hfill
	\begin{minipage}[b]{0.325\textwidth}
		\begin{center}
			\centerline{\includegraphics[width=\columnwidth]{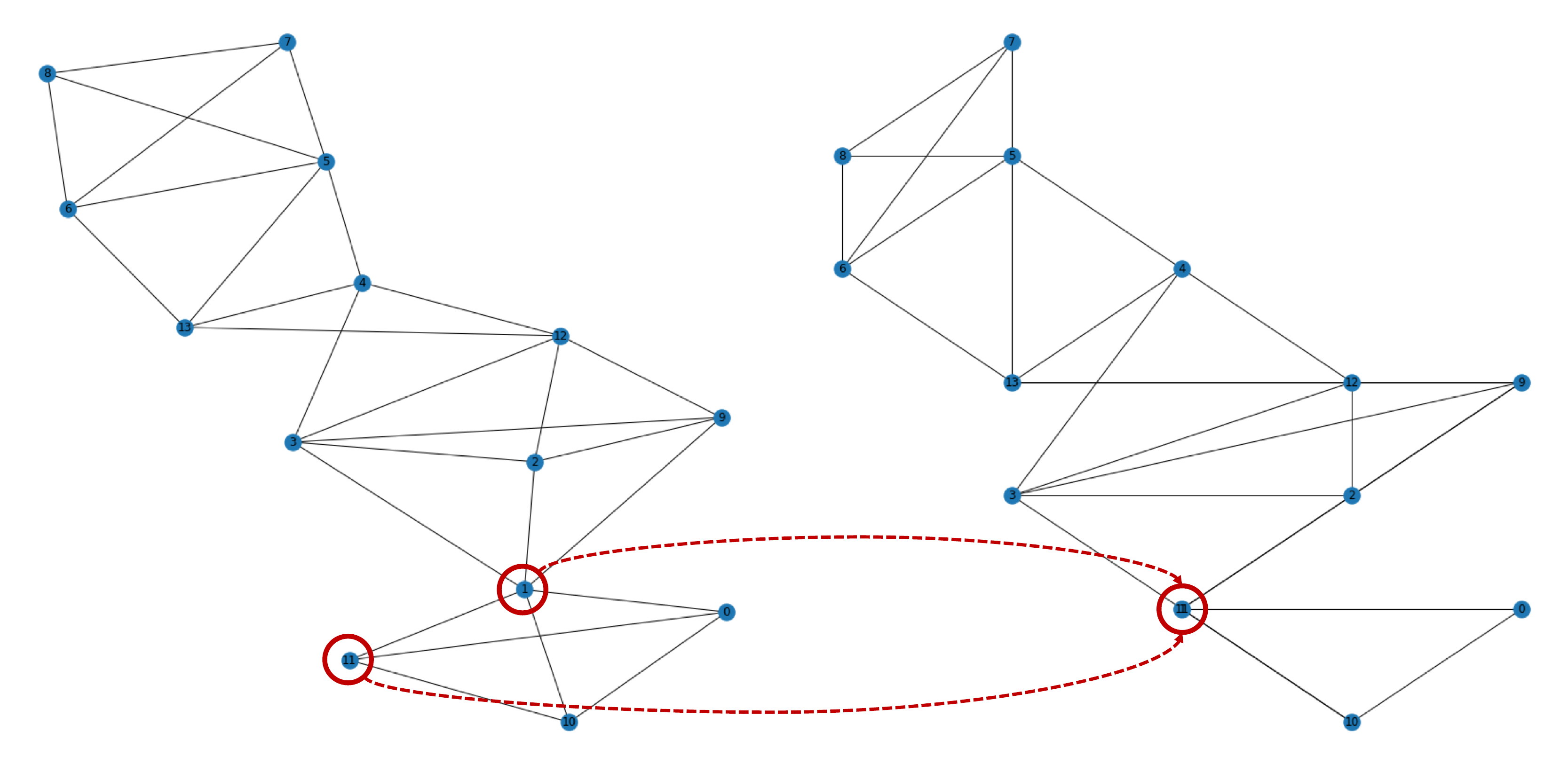}}
			\centerline{\footnotesize (c) PROTEINS}
		\end{center}
	\end{minipage}
	\hfill
	\vspace{-5mm}
	\caption{\footnotesize Illustration of vertex loss on different data sets: In each subfigure, {\bf (left)} before rounding and {\bf (right)} after rounding.}
	\label{fig:vertex_loss}
	% \vspace{-3mm}
\end{figure*}
\begin{figure*}[t]
	\begin{minipage}[b]{0.245\textwidth}
		\begin{center}
			\centerline{\includegraphics[width=\columnwidth]{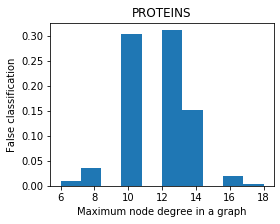}}
%			\centerline{\footnotesize (a) MUTAG}
		\end{center}
	\end{minipage}
	\hfill
	\begin{minipage}[b]{0.245\textwidth}
		\begin{center}
			\centerline{\includegraphics[width=\columnwidth]{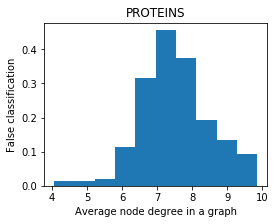}}
%			\centerline{\footnotesize (b) IMDB-B}
		\end{center}
	\end{minipage}
	\hfill
	\begin{minipage}[b]{0.245\textwidth}
		\begin{center}
			\centerline{\includegraphics[width=\columnwidth]{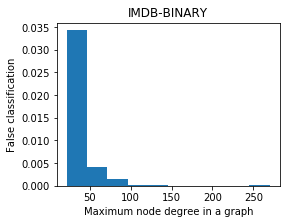}}
%			\centerline{\footnotesize (c) PROTEINS}
		\end{center}
	\end{minipage}
	\hfill
	\begin{minipage}[b]{0.245\textwidth}
		\begin{center}
			\centerline{\includegraphics[width=\columnwidth]{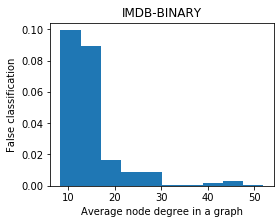}}
%			\centerline{\footnotesize (a) MUTAG}
		\end{center}
	\end{minipage}
	\hfill
	\begin{minipage}[b]{0.245\textwidth}
		\begin{center}
			\centerline{\includegraphics[width=\columnwidth]{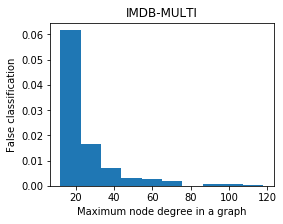}}
%			\centerline{\footnotesize (b) IMDB-B}
		\end{center}
	\end{minipage}
	\hfill
	\begin{minipage}[b]{0.245\textwidth}
		\begin{center}
			\centerline{\includegraphics[width=\columnwidth]{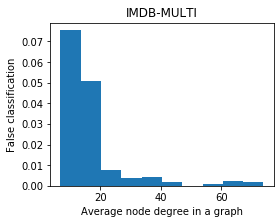}}
%			\centerline{\footnotesize (c) PROTEINS}
		\end{center}
	\end{minipage}
	\hfill
	\begin{minipage}[b]{0.245\textwidth}
		\begin{center}
			\centerline{\includegraphics[width=\columnwidth]{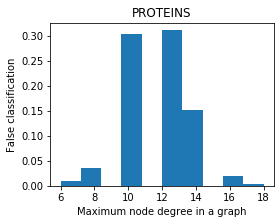}}
%			\centerline{\footnotesize (a) MUTAG}
		\end{center}
	\end{minipage}
	\hfill
	\begin{minipage}[b]{0.245\textwidth}
		\begin{center}
			\centerline{\includegraphics[width=\columnwidth]{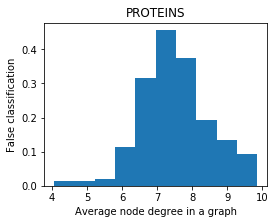}}
%			\centerline{\footnotesize (b) IMDB-B}
		\end{center}
	\end{minipage}
	\hfill	
	\vspace{-5mm}
	\caption{\footnotesize Graph topology in terms of node degree \vs misclassification on different data sets. The y-axis denotes the fractions of misclassification in the graphs with the same node degrees.}
	\label{fig:miscls}
	\vspace{-3mm}
\end{figure*}

\bfsection{Classifier: Multi-Scale Maxout CNNs (MSM-CNNs)}
We apply CNNs to the 3D representations of graphs for classification. As we discussed above, once the node degree is higher than 8, the grid layout cannot fully preserve the topology, but rather tends to form a ball-like compact pattern with larger neighborhood. To capture such neighborhood information effectively, the kernel sizes in the 2D convolution need to vary. Therefore, the problem now boils down to a feature selection problem with convolutional kernels.

Considering these, here we propose using a multi-scale maxout CNN as illustrated in Fig. \ref{fig:mm_cnn}. We use consecutive convolutions with smaller kernels to approximate the convolutions with larger kernels. For instance, we use two $3\times3$ kernels to approximate a $5\times5$ kernel. The maxout \citep{goodfellow2013maxout} operation selects which scale per grid node is good for classification and outputs the corresponding features. Together with other CNN operations such as max-pooling, we can design deep networks, if necessary.

\bfsection{Graph Label Prediction by Majority Voting}
In our experiments we perform 10-fold leave-one-out testing. For simplicity we conduct data augmentation once on all the graphs. Therefore, even a test graph still has multiple grid layouts. By default we use the majority voting mechanism to predict the graph label based on all the predicted labels from its grid layouts. We also try to randomly sample a grid layout and utilize its label as the graph label. We find that statistically there is marginal difference between the two prediction mechanisms. In fact, we observe that the accuracy of graph prediction is very close to that of grid layout prediction.

\section{Experiments}
\bfsection{Data Sets}
We evaluate our method, \ie graph-preserving grid layout + (multi-scale maxout) CNNs, on four medium-size benchmark data sets for graph classification, namely MUTAG, IMDB-B, IMDB-M and PROTEINS. Table \ref{tab:statistics} summarizes some statistics of each data set. Note that the max node degree on each data set is at least 8, indicating that ball-like patterns as discussed in Prop. \ref{prop:1} may occur, especially for IMDB-B and IMDB-M.

\begin{table}[t]
	\caption{\footnotesize Statistics of benchmark data sets for graph classification.}
	\label{tab:statistics}
	\vspace{-2mm}
	\adjustbox{width=\columnwidth}{
		\centering
		\begin{tabular}{c|ccccccc}
			\toprule
			Data Set & \begin{tabular}[c]{@{}c@{}}Num. of\\Graph\end{tabular} & \begin{tabular}[c]{@{}c@{}}Num. of\\Class\end{tabular} & \begin{tabular}[c]{@{}c@{}}Avg.\\Node\end{tabular} & \begin{tabular}[c]{@{}c@{}}Avg.\\Edge\end{tabular} & \begin{tabular}[c]{@{}c@{}}Avg.\\Degree\end{tabular} & \begin{tabular}[c]{@{}c@{}}Max\\Degree\end{tabular} & \begin{tabular}[c]{@{}c@{}}Feat.\\Dim.\end{tabular} \\
			\midrule
			MUTAG & 188 & 2 & 17.93 & 19.79 & 1.10 & 8 & 7 \\
			IMDB-B & 1000 & 2 & 19.77 & 96.53 & 4.88 & 270 & 136 \\
			IMDB-M & 1500 & 3 & 13.00 & 65.94 & 5.07 & 176 & 89 \\
			PROTEINS & 1113 & 2 & 39.06 & 72.82 & 1.86 & 50 & 3 \\
			\bottomrule
		\end{tabular}
	}
%	\vspace{-3mm}
\end{table}

\bfsection{Implementation}
By default, we set the parameters in Alg. \ref{alg:RKK_alg} as $\alpha=1.25, \lambda=1000$. In this paper we do not use rescaling because we find that the initial KK solutions are sufficient to learn good grid layouts. We crop all the grid layouts to a fixed-size $64\times 64$ window.

Also by default, for the MSM-CNNs we utilize three consecutive $3\times3$ kernels in the MSM-Conv, and design a simple network of ``MSM-Conv(64)$\rightarrow$max-pooling$\rightarrow$MSM-Conv(128)$\rightarrow$max-pooling$\rightarrow$MSM-Conv(256)$\rightarrow$global-pooling$\rightarrow$FC(256)$\rightarrow$FC(128)'' as hidden layers with ReLU activations, where FC denotes a fully connected layer and the numbers in the brackets denote the numbers of channels in each layer. We employ Adam \citep{kingma2014adam} as our optimizer, and set batch size, learning rate, and dropout ratio to 10, 0.0001, and 0.3, respectively.

\subsection{Ablation Study}

\bfsection{Effects of $\alpha, \lambda$ on Grid Layout and Classification}
%$\alpha, \lambda$ are two important parameters in our algorithm. 
To understand their effects on grid layout generation, we visualize some results in Fig. \ref{fig:alpha_lambda} using different combinations of $\alpha, \lambda$. We can see that:
\begin{itemize}
	\item From Fig. \ref{fig:alpha_lambda}(a)-(c), the diameters of grid layouts are $5\times5$, $6\times6$, $7\times7$ for $\alpha=1.00, 1.25, 1.50$, respectively. This strongly indicates that a smaller $\alpha$ tends to lead to a more compact layout at the risk of losing vertices.
	\item From Fig. \ref{fig:alpha_lambda}(c)-(e), similarly the diameters of grid layouts are $5\times5$, $6\times6$, $6\times6$ for $\lambda=200, 1000, 5000$, respectively. This indicates that a smaller $\lambda$ tends to lead to a more compact layout at the risk of losing vertices as well. In fact in Fig. \ref{fig:alpha_lambda}(d) node 1 and node 5 are merged together. When $\lambda$ is sufficiently large, the layout tends to be stable.
\end{itemize}
Such observations follow our intuition in designing Alg. \ref{alg:RKK_alg}, and occur across all the four benchmark data sets.

We also test the effects on classification performance. For instance, we generate 21x grid layouts using data augmentation on MUTAG, and list our results in Table \ref{tab:alpha_lambda}. Clearly our default setting achieves the best test accuracy. It seems that for classification parameter $\alpha$ is more important. %, and it should not be set to large numbers. 

\begin{table}[t]
	\caption{\footnotesize Mean accuracy (\%) using different combinations of $\alpha, \lambda$.}
	\label{tab:alpha_lambda}
	\vspace{-2mm}
	\adjustbox{width=\columnwidth}{
		\centering
		\begin{tabular}{c|ccccc}
			\toprule
			Data Set & \begin{tabular}[c]{@{}c@{}}$\alpha=1.00$\\$\lambda=1000$\end{tabular} & \begin{tabular}[c]{@{}c@{}}$\alpha=1.50$\\$\lambda=1000$\end{tabular} & \begin{tabular}[c]{@{}c@{}}$\alpha=1.25$\\$\lambda=1000$\end{tabular} & \begin{tabular}[c]{@{}c@{}}$\alpha=1.25$\\$\lambda=200$\end{tabular} & \begin{tabular}[c]{@{}c@{}}$\alpha=1.25$\\$\lambda=5000$\end{tabular} \\
			\midrule
			MUTAG (21x) & 85.14 & 83.04 & 86.31 & 85.26 & 85.26 \\
			\bottomrule
		\end{tabular}
	}
	\vspace{-3mm}
\end{table}

\bfsection{Vertex Loss, Graph Topology \& Misclassification}
To better understand the problem of vertex loss, we visualize some cases in Fig. \ref{fig:vertex_loss}. The reason for this behavior is due to the small distances among the vertices returned by Alg. \ref{alg:RKK_alg} that cannot survive from rounding. Unfortunately we do not observe a pattern on when such loss will happen. Note that our Alg. \ref{alg:RKK_alg} cannot avoid vertex loss with guarantee, and in fact the vertex loss ratio on each data set is very low, as shown in Table \ref{tab:vertex_loss_ratio}.

\begin{table}[t]
	\caption{\footnotesize Vertex loss ratio (\%) on each data set.}
	\label{tab:vertex_loss_ratio}
	\vspace{-2mm}
	\adjustbox{width=\columnwidth}{
		\centering
		\begin{tabular}{c|cccc}
			\toprule
			Data Set & MUTAG & IMDB-B & IMDB-M & PROTEINS
			\\
			\midrule
			Vertex Loss & 1.06 & 0.99 & 0.40 & 0.90 \\
			\bottomrule
		\end{tabular}
	}
% 	\vspace{-9mm}
\end{table}

\begin{table}[t]
	\caption{\footnotesize Ratios (\%) between vertex loss and misclassification.}
	\label{tab:vertex_loss_miscls}
	\vspace{-2mm}
	\adjustbox{width=\columnwidth}{
		\centering
		\begin{tabular}{c|ccc}
			\toprule
			Data Set & \begin{tabular}[c]{@{}c@{}}$\underline{|\mathcal{G}_{v.l.}|}$\\$|\mathcal{G}_{mis.}|$\end{tabular} & \begin{tabular}[c]{@{}c@{}}$\underline{|\mathcal{G}_{v.l.}\cap\mathcal{G}_{mis.}|}$\\$|\mathcal{G}_{v.l.}|$\end{tabular} & \begin{tabular}[c]{@{}c@{}}$\underline{|\mathcal{G}_{n.v.l.}\cap\mathcal{G}_{mis.}|}$\\$|\mathcal{G}_{n.v.l.}|$\end{tabular}
			\\
			\midrule
			MUTAG (21x) & 1.06 & 20.00 & 16.70 \\
			IMDB-B (3x) & 0.99 & 16.18 & 37.41 \\
			PROTEINS (3x) & 0.90 & 24.32 & 29.89 \\
			\bottomrule
		\end{tabular}
	}
	\vspace{-9mm}
\end{table}

Further we test the relationship between vertex loss and misclassification, and list our results in Table \ref{tab:vertex_loss_miscls} where $\mathcal{G}_{v.l.}$, $\mathcal{G}_{n.v.l.}$, and $\mathcal{G}_{mis.}$ denote the sets of graphs with vertex loss, no vertex loss, and misclassification, respectively, $\cap$ denotes the intersection of two sets, $|\cdot|$ denotes the cardinality of the set, and the numbers in the brackets denote the numbers of grid layouts per graph in data augmentation. From this table, we can deduce that vertex loss cannot be not the key reason for misclassification, because it takes only tiny portion in misclassification and the ratios of misclassified graphs with/without vertex loss are very similar, indicating that misclassification more depends on the classifier rather than vertex loss.

Next we test the relationship between graph topology in terms of node degree and misclassification, and show our results in Fig. \ref{fig:miscls}. As discussed before, a larger node degree is more difficult for preserving topology. In this test we would like to verify whether such topology loss introduces misclassification. Compared with the statistics in Table \ref{tab:statistics}, it seems that topology loss does cause trouble in classification. One of the reasons may be that the variance of the grid layout for a vertex with larger node degree will be higher due to perturbation. Designing better CNNs will be one of our future works to improve the performance.

\bfsection{CNN based Classifier Comparison}
We test the effectiveness of our MSM-CNNs, compared with $(1\times1)$-kernel counterpart and ResNet-50 \citep{he2016deep}, using the same augmented data. On MUTAG in terms of mean test accuracy, MSM-CNNs can achieve 89.34\%, while $(1\times1)$-kernel counterpart and ResNet-50 achieve 86.79\% and 86.78\%, respectively. Similar observations are made on the other data sets. Therefore, the feature selection mechanism seems very useful in graph classification with our grid layout algorithm.

\setlength{\columnsep}{15pt}
\begin{wrapfigure}{r}{.5\linewidth}
	\vspace{-20pt}
	\begin{center}
		\includegraphics[width=\linewidth]{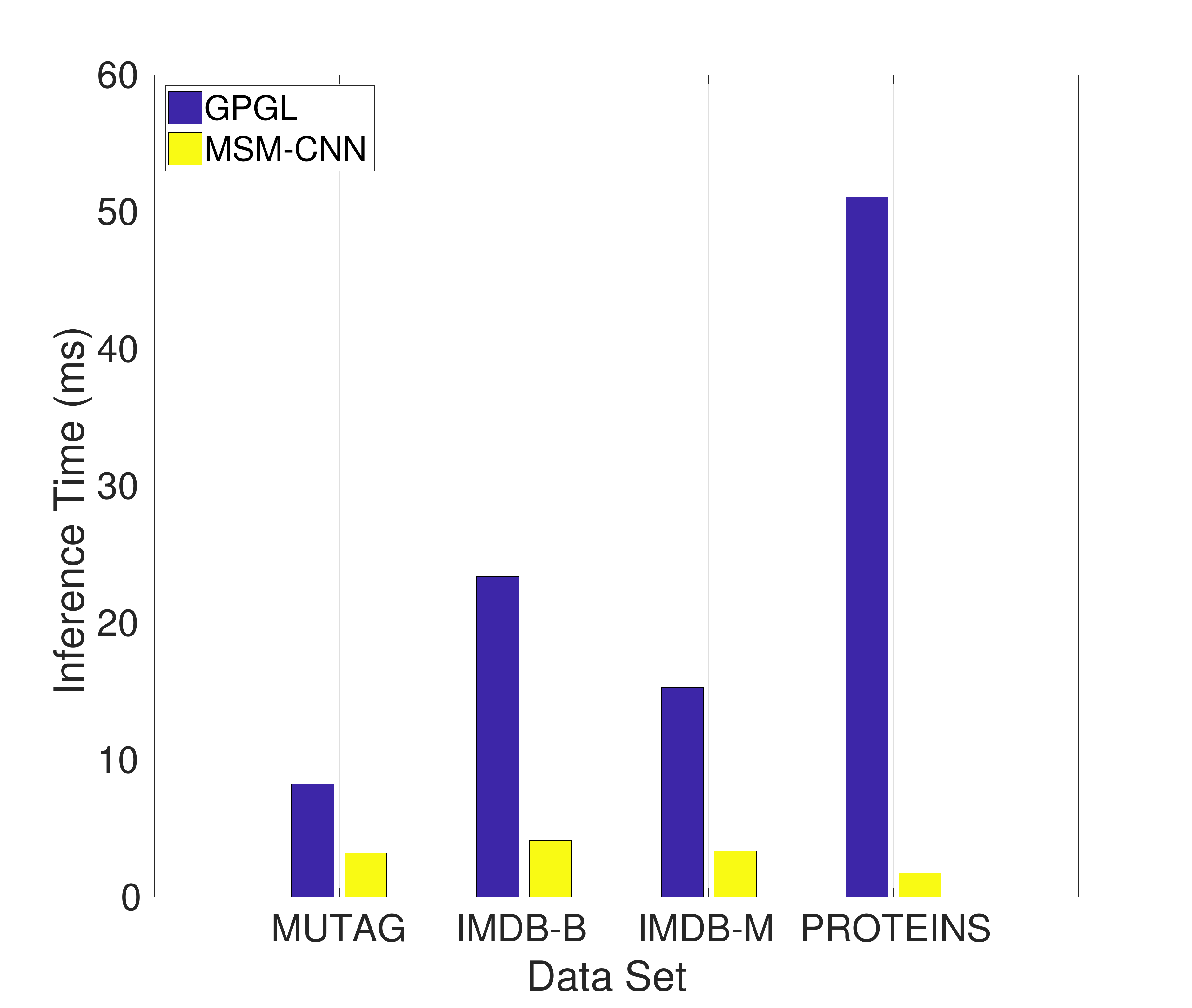}
		\vspace{-5mm}
		\caption{\footnotesize Running time at inference for GPGL and MSM-CNN.}
		\label{fig:running_time}
	\end{center}
	\vspace{-10pt}
\end{wrapfigure}
\bfsection{Running Time}
To verify the running time of our method, we show the inference time in Fig. \ref{fig:running_time} for both GPGL and MSM-CNN, respectively. GPGL is optimized on an Intel Core i7-7700K CPU and MSM-CNN is run on a GTX 1080 GPU. We do not show the training time for MSM-CNN because it highly depends on the number of training samples as well, but should be linearly proportional to inference time, roughly speaking. As we see, the inference time for MSM-CNN is roughly the same, and much less than that for GPGL. The running time of GPGL on PROTEINS is the highest due to its largest average number of nodes among the four data sets, which is consistent with our complexity analysis.

\setlength{\columnsep}{15pt}
\begin{wrapfigure}{r}{.5\linewidth}
	\vspace{-15pt}
	\begin{center}
		\includegraphics[width=\linewidth]{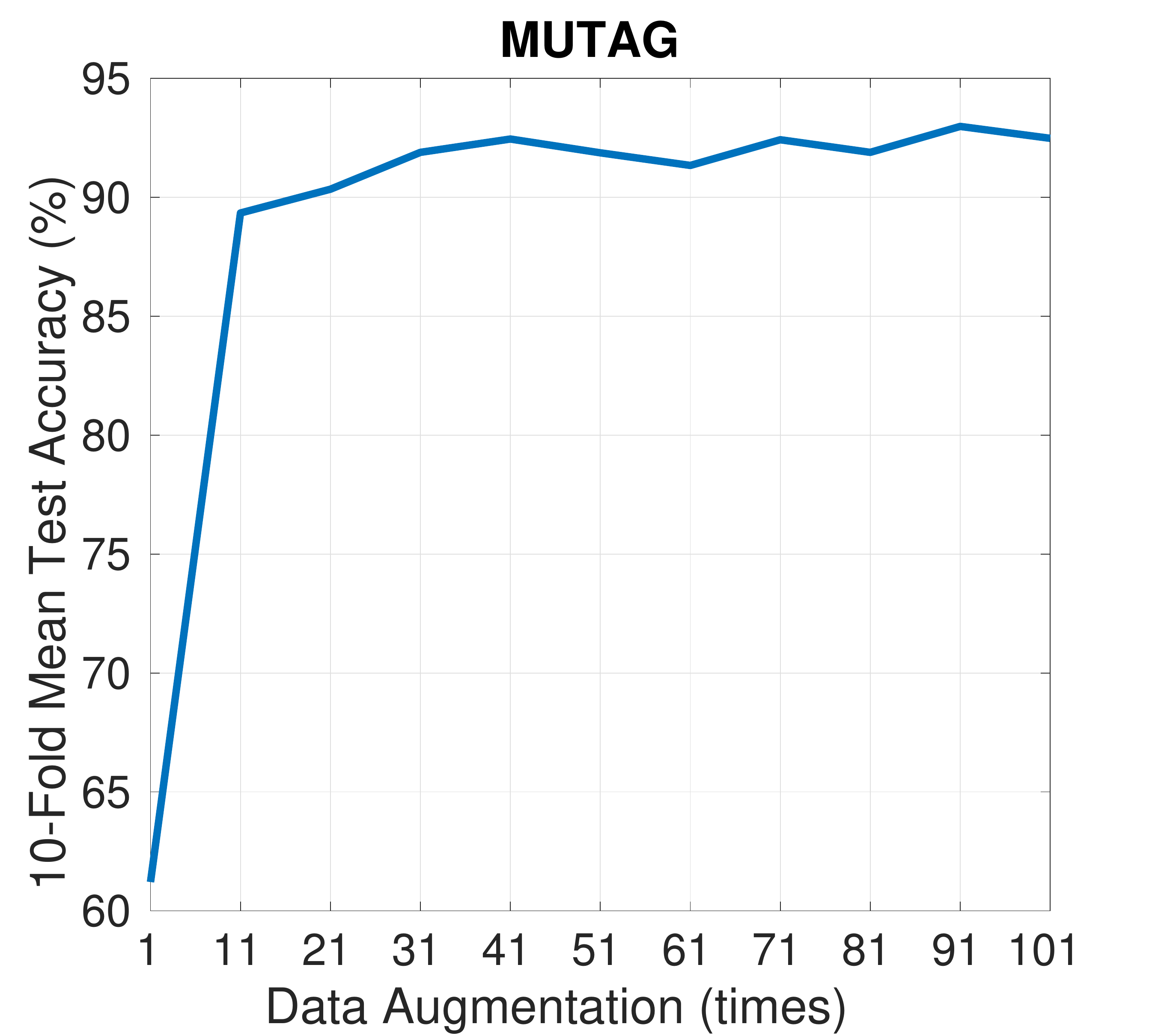}
		\vspace{-5mm}
		\caption{\footnotesize Illustration of data augmentation on classification.}
		\label{fig:augmentation}
	\end{center}
	\vspace{-10pt}
\end{wrapfigure}
\bfsection{Effect of Data Augmentation using Grid Layouts on Classification}
In order to train the deep classifiers well, the amount of training data is crucial. As shown in Alg.~\ref{alg:RKK_alg}, our method can easily generate tons of grid layouts that effectively capture different characteristics in the graph. Given the memory limit, we demonstrate the test performance for data augmentation in Fig.~\ref{fig:augmentation}, ranging from 1x to 101x with step 10x. As we see clearly, data augmentation can significantly boost the classification accuracy on MUTAG, and similar observations have been made for the other data sets.

\subsection{State-of-the-Art Comparison}
To do a fair comparison for graph classification, we follow the standard routine, \ie 10-fold cross-validation with random split. We directly cite the numbers from the leader board (as of 09/05/2019) at \url{https://paperswithcode.com/task/graph-classification} on each data set without reproducing these results. All the comparisons are summarized in Table \ref{tab:sota}, where the numbers in the brackets denote the numbers of grid layouts per graph in data augmentation.

In general, our method can match or even outperform the state-of-the-art on the four data sets, achieving two rank-1, one rank-2 and one rank-3 in terms of mean test accuracy. The small variances indicate the stability of our method. In summary, such results demonstrate the empirical success of our method on graph classification.

\makesavenoteenv{tabular}
\makesavenoteenv{table}
\begin{table}[t]
	\caption{\footnotesize State-of-the-art test accuracy (\%) comparison.}
	\label{tab:sota}
	\vspace{-2mm}
	\adjustbox{width=\columnwidth}{
		\centering
		\begin{tabular}{c|cccc}
			\toprule
			Data Set & Rank-1 & Rank-2 & Rank-3 & {\bf Ours} \\
			\midrule
% 			MUTAG (101x)\footnote{In the order of ranks: \citep{haonan2019graph, niepert2016learning, ivanov2018anonymous}} & {\bf 91.20} & 88.95 & 87.87 & 90.42$\pm$5.59 \\
            MUTAG (101x)\footnote{In the order of ranks (same as follows): \citep{haonan2019graph, niepert2016learning, ivanov2018anonymous}} & 91.20 & 88.95 & 87.87 & {\bf 92.48$\pm$5.30} \\
			IMDB-B (21x)\footnote{\citep{ivanov2018anonymous, morris2019weisfeiler} \citep{morris2019weisfeiler}} & 74.45 & 74.20 & 73.50 & {\bf 74.90$\pm$4.01} \\ 
			IMDB-M (5x)\footnote{\citep{morris2019weisfeiler, xinyi2018capsule, morris2019weisfeiler}} & 51.50 & 50.27 & 49.50 & {\bf 51.99$\pm$2.36}\\ 
			PROTEINS (5x)\footnote{\citep{haonan2019graph, li2019semi, morris2019weisfeiler}} & {\bf 77.90} & 77.26 & 76.40 & 76.44$\pm$1.56 \\ 
			\bottomrule
		\end{tabular}
	}
%	\vspace{-3mm}
\end{table}

\section{Conclusion}
In this paper we answer the question positively that CNNs can be used directly for graph applications by projecting graphs onto grids properly. To this end, we propose a novel graph drawing problem, namely graph-preserving grid layout (GPGL), which is an integer programming to learn 2D grid layouts by minimizing topology loss. We propose a regularized Kamada-Kawai algorithm to solve the integer programming and a multi-scale maxout CNN to work with GPGL. We manage to demonstrate the success of our method on graph classification that matches or even outperforms the state-of-the-art on four benchmark data sets. As future work we are interested in applying this method to real-world problems such as point cloud classification.

% \newpage
{\small
	\bibliographystyle{aaai}
	\bibliography{egbib}
}

\end{document}